\documentclass{article}

\usepackage[preprint]{tmlr}

\usepackage[utf8]{inputenc}
\usepackage[T1]{fontenc}
\usepackage{hyperref}
\usepackage{url}
\usepackage{array}
\usepackage{booktabs}
\usepackage{siunitx}
\usepackage{amsfonts}
\usepackage{amsmath}
\usepackage{amssymb}
\usepackage{amsthm}
\usepackage{nicefrac}
\usepackage{microtype}
\usepackage{graphicx}
\graphicspath{{figs/}}
\usepackage{setspace}
\usepackage{placeins}

\newtheorem{proposition}{Proposition}
\theoremstyle{definition}
\newtheorem{example}{Example}

\title{What a Deletion Certificate Covers, and Where It Expires:\\
Auditable Removal from a Support-Vector Memory\thanks{Code and evidence: \url{https://github.com/VyLabs-AI/sv-attention}.}}

\author{\name Vishwajith Ramesh \email vish@vylabs.ai \\
        \addr Vy Labs, Inc., Cupertino, California, USA}

\def\month{09}
\def\year{2026}

\hypersetup{
  hidelinks,
  pdfauthor={Vishwajith Ramesh},
  pdftitle={What a Deletion Certificate Covers, and Where It Expires: Auditable Removal from a Support-Vector Memory},
  pdfsubject={Public preprint aligned with journal scientific content}
}

\begin{document}

\maketitle

\begin{abstract}
We study deletion in a context memory that fits a support-vector boundary around stored keys and uses the resulting nonnegative coefficients to weight their values. An exactly zero coefficient lets us remove a key without changing the current normalized readout. In a three-key construction, however, a later admission makes the discarded key active in a solve over the full history. For a positive-weight key, a decremental update targets a refit on the remaining keys at the original coefficient cap, under its stated assumptions. Earlier standalone comparisons show close typical gate-score agreement, while similar keys with different values can yield larger readout differences. The journal extension tests chained deletions and admissions, checking coefficient feasibility and optimality conditions in both returned states and fresh references. The initial audit completes 66 of 80 trajectories and exposes numerical failures. A separately checked wrapper completes all \(5{,}120\) scheduled operations at unchanged acceptance thresholds, with eight initialization rescues and 84 update rescues. Its largest readout difference is \(0.5273\%\) of retained-value range. These finite results support explicit acceptance and rescue rules for the evaluated memory edits. They do not extend the current-state removal guarantee to arbitrary future admissions.

\end{abstract}

\section{Introduction}

As a context grows, a memory has to decide which entries to keep. One
possibility is to remove entries that contribute nothing to its output.
We study this choice in a memory where each stored key has an explicit
nonnegative weight. The weights determine how stored values contribute to
the normalized output, or \emph{readout}, so we can check exactly what
removing an entry changes. We call an entry with zero weight a
\emph{reserve key}. Under the assumptions below, removing it preserves
the solved readout; we refer to this guarantee as a \emph{certificate}.
The question is whether the same entry will remain unnecessary as the
memory receives more information.

We obtain the weights from a one-class support-vector fit around the stored
keys~\citep{ocsvm,svdd}, while the query determines which keys resemble the
information being requested (Section~\ref{sec:gate}). For example, in an
hourly ICU memory, a key can encode standardized vital signs and its value
can hold an associated measurement or event label. A query describing a
pattern of vital signs would combine values from similar hours. Repeated,
nearly identical hours may be redundant for the fitted boundary, whereas an
unusual hour may affect it. This geometric distinction does not tell us
whether an hour is medically important or its content is private. A system
using this memory would therefore need to check what clinical information
survives selection. We take a first step by measuring retention of
threshold-defined event hours in Appendix~\ref{app:selection}; clinical
outcomes are not evaluated.

We use established support vector data description and maintained
support-vector optimization~\citep{svdd,cp2000} to remove entries from
this memory. For an entry with positive weight, we update the fit and
compare it with a fresh solve on the surviving keys, keeping the original
limit on each weight. This gives us a reference against which to check
the numerical update. We examine its agreement in \(1{,}200\) deletion
trials across synthetic, clinical and learned keys, including the cases
where score and readout agreement differ. Separately, a three-key
construction in Section~\ref{sec:expiry} shows what happens to a reserve
key when the context changes. The construction lets us examine later
admissions without confusing their effect with numerical errors in the
deletion update. Appendices~\ref{app:selection} and~\ref{app:training}
report our selection and trainability experiments.

The journal extension tests the numerical update through chained deletions
and admissions. The initial sequential audit exposes returned states and fresh
references that fail independently recomputed feasibility and optimality
checks. A separate wrapper then completes all \(5{,}120\) scheduled operations
with unchanged acceptance thresholds and explicit refit rescues. It checks a
candidate before publishing it and retains the previous valid state if rescue
fails. This gives an inference-time memory a tested procedure for accepting or
refusing an edit. Section~\ref{sec:journal-sequential-checks} reports both runs,
including the initial failures and the remaining readout differences. The
wrapper checks the inherited solver; it does not repair or prove its update
algorithm.

\section{Related Work}\label{sec:related}

Linear attention, fast-weight programmers, DeltaNet, mesa layers, Titans,
and selective state-space models compress context into recurrent or
test-time-trained state~\citep{lina,fwp,deltanet,mesa,mesanet,titans,mamba}.
The delta-attention study examines record omission in recurrent memory through
a transport criterion and replay reference~\citep{kdaomission}. The frozen-Gemma
study installs the support-vector gate in a pretrained model and evaluates
memory edits behaviorally~\citep{gemmamemory}.

Cache selectors choose which explicit entries to retain. H2O uses
recency and accumulated attention, with an approximation guarantee under its
assumptions~\citep{h2o}; SnapKV uses attention in a prompt observation
window~\citep{snapkv}. VeriCache maintains a full reference cache alongside its
pruned cache~\citep{vericache}. Each method provides a way to relate the
entries it keeps to the output it aims to preserve. We approach the problem
through a key's exactly zero coefficient in the current support-vector solve,
which gives us an algebraic condition for removing that key.

Deletion can also be defined against retraining without the target
data~\citep{caoyang,unlearn}. ECO uses decremental support-vector machinery for
that reference in a classifier head~\citep{eco}; here the target is an
inference-time refit on the remaining keys. OptNet and differentiable convex layers
obtain gradients through a forward problem's optimality
conditions~\citep{optnet,dcl}. Our maintained path reuses the inverse already
stored for its updates, so the same maintained system supports both memory
changes and differentiation. Our experiments below examine deletion in
this system.

\section{Methods}\label{sec:methods}

\subsection{A support-vector gate over context keys}\label{sec:gate}

We represent each token in the context by a key \(x_i\in\mathbb{R}^d\), which
serves as its address, and a value \(v_i\in\mathbb{R}^{d_v}\), which supplies
content to the output. A query \(q\) determines how closely the request
matches each key, and \(\alpha_i\geq0\) is the coefficient that our gate
assigns to key \(i\). We combine these two weights in the normalized readout
\begin{equation}
o(q)=
\frac{\sum_i \alpha_i\kappa(q,x_i)v_i}
     {\sum_i \alpha_i\kappa(q,x_i)},\qquad
\kappa(a,b)=\exp(-\lVert a-b\rVert^2/\sigma^2),
\end{equation}
where \(\kappa\) is a radial-basis-function similarity under which nearby
vectors score higher and \(\sigma\) sets the distance scale. To obtain the
coefficients, we solve support vector data description (SVDD), a convex
problem that fits a soft
boundary around the keys~\citep{svdd}:
\begin{equation}
\min_{\alpha}\;
\alpha^\top K\alpha-\operatorname{diag}(K)^\top\alpha
\quad\text{s.t.}\quad
\sum_i\alpha_i=1,\qquad 0\leq\alpha_i\leq C,
\label{eq:svdd}
\end{equation}
where \(K_{ij}=\kappa(x_i,x_j)\) and \(C\) is the cap, the largest coefficient
any single key may receive. We use the sparsity parameter \(\nu\) only once, to
set \(C=1/(\nu n_0)\) at a declared initial context size \(n_0\); deletion and
refit never change \(C\) when the number of keys changes. Because the
coefficients must sum to one, a problem with \(n\) keys is feasible only when
\(nC\geq1\).

That is, the coefficients divide a unit budget among the stored keys, with
each key limited to \(C\). We can remove a zero-weight key without changing
this allocation, whereas deleting a positive-weight key requires an update
to the fit. Note that a key's final share of the readout also depends on
query similarity and normalization; \(C\) limits its coefficient rather
than its final share.

The optimality conditions of Equation~\ref{eq:svdd} sort the keys into three
groups according to where their coefficient sits: the margin set
\(S=\{i:0<\alpha_i<C\}\), the upper-bound set \(E=\{i:\alpha_i=C\}\), and the
reserve set \(R=\{i:\alpha_i=0\}\). A reserve key contributes zero to both
the numerator and denominator of the readout, which is why we can remove it
from the solved state.

The fixed cap limits the number of keys that can be removed while preserving
feasibility. For \(\nu=0.5\) and \(n_0=12\), the cap is
\(C=1/6\). The retained coefficients can sum to one only while at least six keys
remain, so a request to delete five of the twelve keys can be answered and a
request to delete seven cannot: no assignment of at most \(1/6\) to each of five
survivors reaches a total of one. In general, a context of \(n\) keys admits at
most \(n-\lceil 1/C\rceil\) deletions before the retained problem becomes
infeasible at that cap. To delete more keys, we would have to change the cap
and rebuild the boundary, and would then be comparing against a different
reference problem.

We compare numerical solutions using the implementation's scalar \emph{gate
score},
\begin{equation}
s(x)=2\sum_i\alpha_i\kappa(x,x_i)-\rho,
\label{eq:gate-score}
\end{equation}
Writing \(f\) for the objective in Equation~\ref{eq:svdd}, the implementation
recovers \(\rho\) using the Lagrangian convention
\[
L(\alpha,\rho,\mu,\eta)=f(\alpha)+\rho(\mathbf1^\top\alpha-1)
-\mu^\top\alpha+\eta^\top(\alpha-C\mathbf1).
\]
Thus the Karush--Kuhn--Tucker (KKT) conditions give
\(2(K\alpha)_i-K_{ii}+\rho=0\) on a free margin key, so
\(\rho=K_{ii}-2(K\alpha)_i\). The factor two comes from differentiating the
quadratic term in \(f\). The nonnegative multipliers \(\mu,\eta\) enforce the
lower and upper bounds. The QP reference recovers this offset by averaging over margin keys, with a
most-interior-key fallback when none are free.

This offset is related to the SVDD radius, but Equation~\ref{eq:gate-score}
keeps the score convention used in the recorded experiments. With feature
center \(c=\sum_i\alpha_i\phi(x_i)\), let \(q_\alpha=\alpha^\top K\alpha\).
At an exact margin solution the squared radius is
\(R_{\rm svdd}^2=\rho+q_\alpha\), and the geometric inside-boundary slack is
\[
m(x)=R_{\rm svdd}^2-\lVert\phi(x)-c\rVert^2
=\rho-\kappa(x,x)+2\sum_i\alpha_i\kappa(x,x_i).
\]
Consequently \(s(x)=m(x)+\kappa(x,x)-2\rho\). For this unit-diagonal RBF
kernel the SVDD boundary has score \(1-2\rho\), not necessarily zero.
A reported score \emph{difference} of \(10^{-9}\) therefore describes an
offset-adjusted kernel-similarity difference, rather than a probability or
a direct distance error. We interpret its size by comparing it with the
reference score's spread on the same probes
(Section~\ref{sec:verification}); an individual score near zero does not
have a universal boundary meaning.

The gate score also differs from the readout \(o(q)\): it contains no values
\(v_i\). Redistributing coefficients between identical keys preserves the
weighted kernel sum, so this score lets us compare solutions even when
their coefficients differ. We also measure readout differences to check
what those solutions produce when values are included.

\subsection{Reference states and the scope of the certificate}\label{sec:contract}

To check a deletion, we need to specify the state we expect afterward. A
\emph{refit} re-solves an existing storage problem after selected entries have
been removed from it. Rebuild, replay, and causal regeneration reach other
states and are defined and audited in the delta-attention study~\citep{kdaomission}. The fixed-cap retained-key refit is the reference for active-key deletion. We also
distinguish an \emph{admission}, a token that arrives after the current solve,
from the keys already in the solved context.

The propositions below describe exact removal under their stated
assumptions. In a numerical implementation, we also need to check whether
the solver returned an acceptable state and how far it differs from the
reference at declared tolerances. A streaming policy adds another
complication: it removes entries repeatedly while the context changes.
We compare such a policy with a never-pruned full-history trajectory.
Table~\ref{tab:claims} collects these references and the numerical paths
used for each.

\begin{table}[t]
\centering
\small
\caption{References and numerical paths. Exact statements retain the assumptions of
their propositions; numerical results require separate acceptance and
agreement checks.}
\label{tab:claims}
\begin{tabular}{@{}>{\raggedright\arraybackslash}p{0.17\linewidth}>{\raggedright\arraybackslash}p{0.27\linewidth}>{\raggedright\arraybackslash}p{0.22\linewidth}>{\raggedright\arraybackslash}p{0.26\linewidth}@{}}
\toprule
claim & reference & arithmetic path & not covered \\
\midrule
reserve removal & current fixed-\(C\) state, by algebra & fp64 residual measured & equivalence after future admissions \\
active deletion & retained-key refit at the same \(C\) (Prop.~\ref{prop:decrement}, unique optimum) & maintained fp64 & changed-\(C\) refit; re-encoding; deletion from weights \\
streaming pruning & never-pruned full-history trajectory & maintained fp64 & certified equality; measured only \\
batched training & none (empirical policy) & ridged, packed fp32 autograd & decrement and stored-inverse guarantees \\
\bottomrule
\end{tabular}
\end{table}

We use different solvers for verification and training. For verification,
we update the existing solution in fp64 after each single-key change and
periodically recompute the stored inverse from scratch to control
accumulated rounding. During training, we solve many coefficient problems
in parallel with a stabilized fp32 approximation and a \(10^{-3}\)
diagonal ridge (Appendix~\ref{app:training}). This path is differentiable,
but it does not implement the maintained deletion update. To audit a
trained model, we freeze its keys and re-solve them in fp64, then check the
returned states and their agreement with the reference.

\subsection{Deleting an active key: the decremental update}\label{sec:decrement}

To delete a positive-weight key, we must redistribute its coefficient
among the surviving keys. We use the reverse update of
\citet{cp2000}, which decreases the target coefficient to zero while maintaining the
optimality conditions. For fixed margin and upper-bound sets \(S,E\), those
conditions and the unit-sum constraint form one coupled linear system:
\begin{equation}
\underbrace{\begin{bmatrix}0&\mathbf1^\top\\\mathbf1&2K_{SS}\end{bmatrix}}_{M}
\begin{bmatrix}\rho\\\alpha_S\end{bmatrix}
=\begin{bmatrix}1-C|E|\\\operatorname{diag}(K)_S-2C K_{SE}\mathbf1\end{bmatrix}.
\label{eq:kkt-system}
\end{equation}
The border enforces the coefficient sum while the lower rows keep every free
key stationary. Updating coefficients alone could violate either condition;
\(M\) determines how the offset and free coefficients must move together.
The maintained inverse \(\mathcal R=M^{-1}\) supplies these update directions
without a new full solve at each step. It expands or contracts when keys move
between \(S,E,R\), provided the margin system remains nonsingular.
In the earlier standalone audit, we fell back to the precomputed refit
at the same cap when the maintained update did not complete
(Appendix~\ref{app:protocols}).

\subsection{Verification protocol}\label{sec:protocol}

We compared the decrement with a refit on the remaining keys at identical \(C\) on
four regimes: Gaussian keys, redundant near-duplicates, real hourly clinical
vitals from MIMIC-IV~\citep{mimiciv}, and keys produced by a trained recall
model. Each regime contributed \(300\) attempted trials, for \(1{,}200\) in
total. Each trial removed two keys and evaluated the gate on three probe sets
(the retained keys, the removed keys, and the retained keys perturbed by
Gaussian noise of scale \(0.1\)). As a comparator, we evaluated \emph{coefficient decay},
which leaves the target keys in place, multiplies their coefficients by
\(0.01\), and does not re-solve.

We computed the reference refit before attempting the decrement on each
trial, allowing us to use it both for comparison and as a fallback. We
report whether a result was returned and by which path, followed by how
closely the completed decrements agreed with the reference. This keeps
trials that needed a fallback visible alongside the agreement results.
For each trial, we measured the largest absolute gate-score difference
over the declared probe sets; the measurement samples agreement on those
sets. Appendix~\ref{app:protocols} gives the solver configuration and the
diagnostics recorded per trial.

\subsection{Use of AI assistance}

OpenAI Codex assisted with additional experiment design, code implementation,
execution and analysis, and manuscript preparation, including claim and prose
revisions. This assistance included content preparation beyond spelling and
copyediting.

\section{Results}\label{sec:results}

\subsection{Removing a reserve key leaves the solved readout unchanged}\label{sec:reserve-results}

\begin{proposition}[Point-in-time reserve certificate]\label{prop:reserve}
Let \(\alpha^\star\) solve Equation~\ref{eq:svdd} for a fixed context and fixed
\(C\). If \(\alpha^\star_j=0\), then evicting key \(j\) without re-solving
leaves the objective value and the normalized readout \(o(q)\) unchanged for
every query with nonzero denominator. The restricted coefficient vector is
also optimal for the retained-key problem.
\end{proposition}

\begin{proof}
Restricting \(\alpha^\star\) to the retained indices preserves the coefficient
sum and bounds, and every objective term involving \(j\) vanishes because
\(\alpha^\star_j=0\). Conversely, any better retained-key solution could be
extended by a zero coefficient at \(j\), contradicting optimality of
\(\alpha^\star\). The numerator and denominator of the readout likewise lose
only a zero-weight term.
\end{proof}

In the implementation, we classify reserve coefficients at a declared
tolerance and record the residual change in the readout. For example, the
residual in Example~\ref{ex:future} is \(4.6\times10^{-9}\), where
Proposition~\ref{prop:reserve} gives zero in exact arithmetic.

\subsection{A reserve key can become active after a later admission}\label{sec:expiry}

We next fit a boundary to three keys and add a fourth key far from the
group. This moves the boundary enough that a key previously assigned zero
weight becomes active. We compare the two histories below: one that kept
the reserve key and one that discarded it.

\begin{example}[A reserve key reactivated by a later token]\label{ex:future}
Take the similarity of Section~\ref{sec:gate} with keys \((-1,0)\), \((1,0)\),
and \((0,0.9)\), cap \(C=1\), and width \(\sigma=5\). The third key is reserve:
its coefficient is exactly zero. For the query \(q=(0.25,0.1)\) and scalar
values \((1,2,9)\), removing the third key moves the gate score by less than
\(10^{-12}\) and the readout by \(4.6\times10^{-9}\), consistent with the
current-state certificate up to the recorded numerical residual. Admitting one
more token at
\((0,-10)\) and solving over the full history gives the formerly reserve key a
coefficient of \(0.033\), so it becomes active, and the gate score of the state that
pruned it differs from the full-history reference by \(4.1\times10^{-3}\) over
the declared probes.
\end{example}

\begin{proposition}[No future-equivalence guarantee from reserve status alone]
\label{prop:not-future}
Membership in the reserve set of the current solved problem is insufficient to
guarantee equivalence between a previously pruned state and the full-history
reference after arbitrary future admissions.
\end{proposition}

\begin{proof}
Example~\ref{ex:future} is a counterexample: the same key has zero coefficient
before the admission and positive coefficient afterward, and the two gate
scores then differ.
\end{proof}

Blocking the faraway point would prevent this particular sequence. To
justify a general admission rule, however, we would need to prove that it
preserves the relevant optimum. Another possibility is to retain enough
information to reconstruct the reference when its support set changes.
We have not established a sufficient admission rule or the minimal storage
needed for reconstruction.

\subsection{The fixed-cap target for active-key deletion}\label{sec:decrement-results}

\begin{proposition}[Fixed-\(C\) decrement target]\label{prop:decrement}
Assume the retained-key problem in Equation~\ref{eq:svdd} has a unique
coefficient optimum. If the algebraic reverse path completes while maintaining
the optimality conditions and drives the deleted coefficient to zero, its
restricted coefficient vector equals a fresh retained-key solution of
Equation~\ref{eq:svdd} under the same \(C\).
\end{proposition}

\begin{proof}
At completion, removing the zero target coefficient leaves a retained-key state
satisfying the original constraints, stationarity, valid multiplier signs, and
complementarity under the unchanged cap \(C\). These conditions are sufficient
for the convex problem, and uniqueness identifies the coefficient vector with
the retained-key refit.
\end{proof}

The uniqueness assumption needs care when keys resemble one another.
With exact duplicates, weight can move between identical keys without
changing the objective or weighted kernel sum. Their coefficients can
therefore have multiple optima. Distinct keys under the Gaussian RBF
kernel give a positive-definite Gram matrix and a unique optimum, but
near-duplicates can make the problem ill-conditioned. Two numerical
solvers may then return different approximate coefficients with nearly
equal objective values. Those near-ties do not imply multiple exact optima.

Now suppose the similar keys carry different values. Moving weight from
one key to another changes which value contributes to the readout, even
if the weighted kernel sum changes very little. For exact duplicates,
the sum and, with the same offset, Equation~\ref{eq:gate-score} are
unchanged, while the readout can differ. This is why we measure both
score and readout differences. Neither coefficient nonuniqueness nor
sampled gate-score agreement establishes output equality. Checking whether
the states are valid also requires the actual coefficients and stored
offsets from both the returned and reference states.

\subsection{Numerical verification against the fixed-cap reference}\label{sec:verification}

We first compare the decrement with the refit, then examine a projection
intended to reduce readout disagreement. Table~\ref{tab:evidence-map}
describes which trials enter each experiment and how we normalize their
readout differences. The resulting percentages refer to different
measurements and should be interpreted within their respective experiments.

\begin{table}[t]
\centering\small
\caption{Evidence map: populations, execution records, and readout denominators.}
\label{tab:evidence-map}
\begin{tabular}{@{}>{\raggedright\arraybackslash}p{0.21\linewidth}>{\raggedright\arraybackslash}p{0.35\linewidth}>{\raggedright\arraybackslash}p{0.36\linewidth}@{}}
\toprule
Experiment & Population and state record & Readout measurement \\
\midrule
Standalone (Sec.~\ref{sec:verification}) &
1,200 requests; 1,199 maintained returns and one refit fallback. &
Scalar values on each regime's worst gate-score trial; readout difference divided by that trial's readout range. \\
\addlinespace
Projection (Sec.~\ref{sec:tiebreak}) &
Same attempted trials; each tolerance row uses trials with both projections returned. &
Fixed four-dimensional values; maximum across trials of each trial's probe-wise difference divided by its reference readout range. \\
\bottomrule
\end{tabular}
\end{table}

Across the earlier \(1{,}200\) trials, with \(300\) per regime, we obtained
a result for every request (Table~\ref{tab:deletion}). The maintained
decrement completed \(1{,}199\) times. In one clinical trial, the margin
set was empty and we used the refit already computed for that trial as a
fallback. Note that we had paid for this reference solve as part of the
audit; an operational system would have to account for its cost when a
decrement failed.

We archived scalar differences, statuses, and condition numbers for each
trial, but did not save complete key/Gram matrices, coefficient vectors,
or actual offsets. This prevents independent retrospective state
validation of either the returned solution or its numerical reference.
We can recover sampled disagreement and execution status from these
records, but cannot use them to establish whether the original states
were valid or invalid.
The sequential experiment in
Section~\ref{sec:journal-sequential-checks} records and checks actual states.
Its results do not retrospectively validate or invalidate those earlier trials.

\begin{table}[t]
\centering
\small
\caption{Fixed-\(C\) deletion audit over four regimes with \(300\) attempted
trials each. ``Decrement'' counts trials completed by the maintained update; the
remaining trial returned the precomputed reference by refit fallback. Partition match
counts decrement trials in which deletion and retained-key refit assigned every
retained key to the same coefficient group; its denominator is the decrement
column, and it is a diagnostic; the criterion is gate-score agreement. Gate-score errors are the
median, 95th percentile, and worst case over decrement trials.}
\label{tab:deletion}
\begin{tabular}{@{}l c c c c@{}}
\toprule
regime & returned & \shortstack{by\\decrement} & \shortstack{partition\\match} &
\shortstack{gate-score difference\\median / p95 / worst} \\
\midrule
Gaussian & 300/300 & 300/300 & 296/300 &
\(2.4\times10^{-9}\) / \(1.2\times10^{-7}\) / \(6.4\times10^{-3}\) \\
redundant & 300/300 & 300/300 & 156/300 &
\(5.7\times10^{-7}\) / \(1.6\times10^{-4}\) / \(9.6\times10^{-4}\) \\
MIMIC-IV & 300/300 & 299/300 & 293/299 &
\(1.6\times10^{-9}\) / \(1.2\times10^{-7}\) / \(1.1\times10^{-2}\) \\
learned keys & 300/300 & 300/300 & 300/300 &
\(4.5\times10^{-13}\) / \(4.6\times10^{-12}\) / \(1.1\times10^{-11}\) \\
\bottomrule
\end{tabular}
\end{table}

We generally obtained close gate-score agreement, with a small number of
larger deviations (Table~\ref{tab:deletion}). For Gaussian and clinical
keys, the 95th percentile was near \(10^{-7}\), while the worst case
reached \(10^{-3}\) to \(10^{-2}\). These larger deviations occurred
alongside differences in the coefficient partitions and nearly equal
objective values. On distinct near-duplicate keys, such behavior indicates
sensitivity to solver tolerances rather than multiple exact optima. By
comparison, the learned-key regime had no partition disagreements and a
worst deviation of \(1.1\times10^{-11}\).

In the redundant regime, the partitions matched in only \(52\%\) of
trials, although the median gate-score error was \(5.7\times10^{-7}\).
The nearly singular similarity matrix makes it possible for small rounding
differences to move coefficients across group boundaries while leaving
the weighted kernel sum almost unchanged. Counting partition matches
alone would therefore give a poor account of the score agreement in this
regime. We return to the effect on values below.

We examined both partition disagreement and the condition number of the
retained Gram matrix to understand the larger deviations
(Table~\ref{tab:tail}). The association with partition disagreement was
stronger. Deviations were weakly correlated with the condition number in
the three regimes with a tail, and the worst decile had the same
conditioning as the remaining trials. In every regime, trials whose
group assignments differed from the refit's contained the largest
deviations, by three to six orders of magnitude over trials whose
assignments agreed.

\begin{table}[t]
\centering
\small
\caption{Diagnostics associated with the deviation tail. Spearman \(r_s\) correlates a
trial's gate-score deviation with its retained-Gram condition number;
``disagreeing'' trials are decrement trials whose coefficient groups differ from
the refit's. Numeric deviation cells are medians unless labeled ``worst.''
The redundant-row text identifies where the ten largest deviations occur;
a dash denotes no reported value.}
\label{tab:tail}
\setlength{\tabcolsep}{4pt}
\begin{tabular}{@{}l c c c c@{}}
\toprule
regime & \(r_s\) & disagreeing & deviation, disagreeing & deviation, agreeing \\
\midrule
Gaussian & \(+0.03\) & \(4/300\) & \(1.4\times10^{-3}\) & \(2.3\times10^{-9}\) \\
redundant & \(-0.08\) & \(144/300\) & ten largest deviations & --- \\
MIMIC-IV & \(-0.14\) & \(6/299\) & \(2.2\times10^{-5}\) & \(1.5\times10^{-9}\) \\
learned keys & \(+0.60\) & \(0/300\) & --- & worst \(1.1\times10^{-11}\) \\
\bottomrule
\end{tabular}
\end{table}

The learned-key regime complicates this interpretation. It had no
partition disagreements and a worst deviation of
\(1.1\times10^{-11}\), yet it was the one regime in which deviation
tracked conditioning (\(r_s=+0.60\)). Thus the association with group
boundaries in the other regimes does not identify a single cause, nor
does it rule out a contribution from conditioning.

To put the score differences in context, we compared them with the
reference score's range over the same probes. On the worst gate-score
trial in each regime, the deviation was about \(3\%\) of that range.
The corresponding scalar readout differences were \(0.2\%\),
\(17.4\%\), and \(5.6\%\) of readout range for Gaussian, redundant,
and clinical keys. The redundant regime was most sensitive because it
paired near-identical keys with unrelated values. Appendix~\ref{app:projection}
gives the exact score-range percentages.

Coefficient decay changes the target weights without allowing the
surviving weights to adjust (Figure~\ref{fig:deletion}). In our comparison, it
differed from the reference by \(2.2\times10^{-2}\) to
\(2.4\times10^{-1}\). Updating the fit produced closer agreement than
this suppression of the original coefficients.

\begin{figure}[t]
\centering
\includegraphics[width=\linewidth]{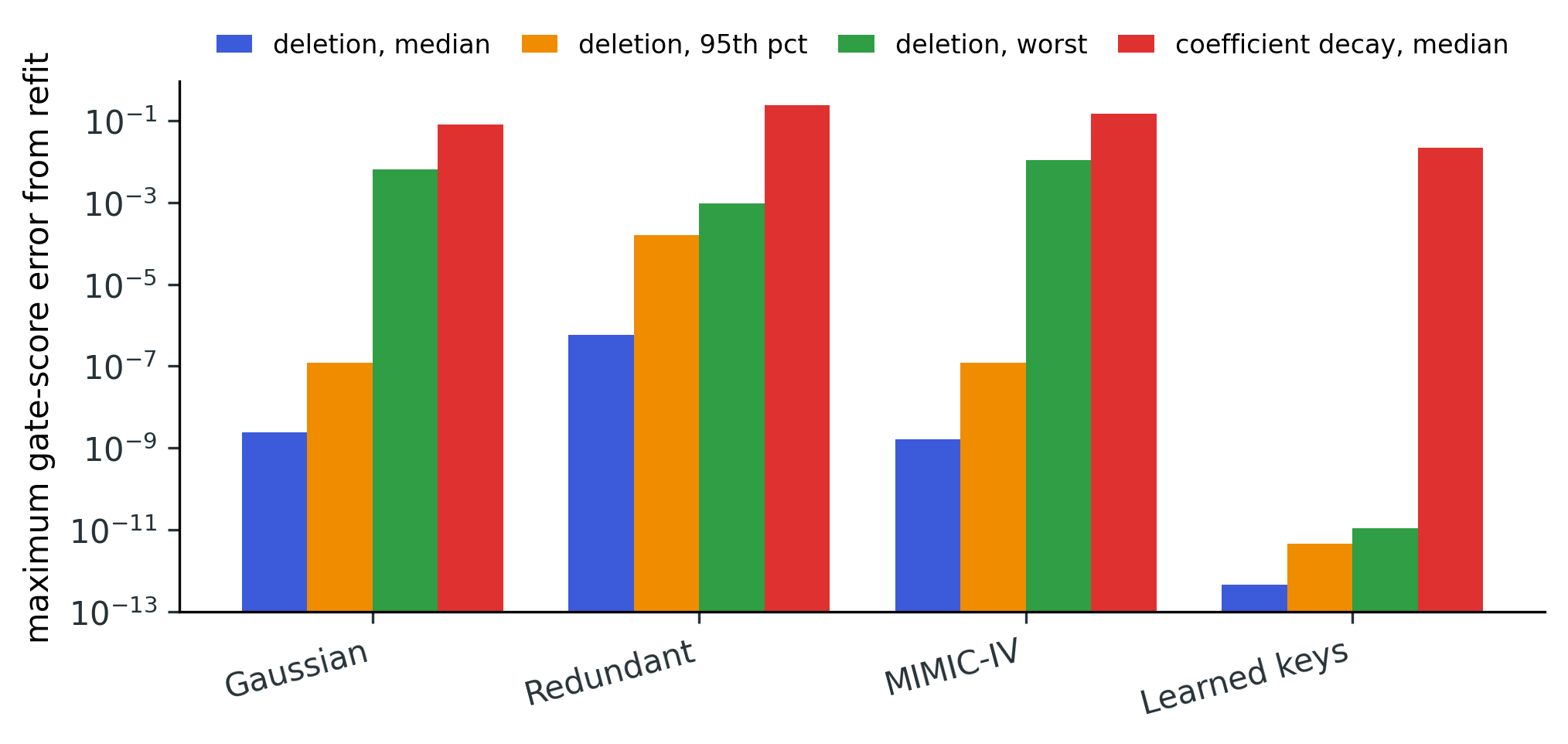}
\caption{Maximum absolute gate-score difference from the fixed-\(C\)
retained-key refit over each trial's probes, in kernel-similarity units.
Maintained deletion and coefficient decay use the same reference. Median,
95th percentile, and worst-case markers expose the tail beyond typical
agreement; Table~\ref{tab:deletion} gives the trial denominators and fallback.}

\label{fig:deletion}
\end{figure}

In the earlier timing comparison, we measured the maintained update alone.
Across \(30\) uniform-random removals at each context size
(\(n\in\{120,256,384,512\}\)), its median latency was
\(0.30\)--\(1.04\) milliseconds on the reference CPU. The median
per-trial ratio of fresh-refit time to maintained deletion was
\(24\)--\(223\times\), increasing with context size on this hardware.
These timings omit independent checks of the returned and reference
states and any rescue needed when an update fails, so they do not give
the cost of a fully checked deletion operation.

We also checked deletions on records represented by learned and clinical
keys. With the learned keys frozen, deleting one fact left the other
seven answerable, and removing then adding a record corrected a single
learned key--value association. With clinical records, the completed
update at fixed \(C\) removed the selected patient's gate weight while
retaining the others. We evaluated these edits through gate scores and
readouts while keeping the encoded keys fixed, as described in
Table~\ref{tab:claims}.

\subsection{Reducing readout disagreement with a minimum-norm projection}\label{sec:tiebreak}

We wanted to see whether choosing coefficients by the same additional
criterion would reduce the readout disagreement. We therefore tested a
minimum-norm tie-break: each method projects its solution to the smallest
coefficient norm allowed within a stated tolerance of its own achieved
objective.
\begin{equation}
P_\varepsilon(\hat\alpha)=\underset{a}{\operatorname{argmin}}\ \lVert a\rVert_2^2
\quad\text{s.t.}\quad
\mathbf1^\top a=1,\quad 0\leq a_i\leq C,\quad
f(a)\leq f(\hat\alpha)+\varepsilon\max\{1,|f(\hat\alpha)|\}.
\label{eq:projection}
\end{equation}
Here \(f\) is the objective in Equation~\ref{eq:svdd}, and \(\hat\alpha\)
is the corresponding method's achieved solution. The positive additive allowance remains
meaningful when \(f\) is negative; multiplying a negative objective by
\(1+\varepsilon\) would instead tighten its cap. The decrement and retained-key
refit use their own objective levels, without seeing the other's. The exact
convex projection has a unique minimum-norm point whenever its set is nonempty.
Appendix~\ref{app:projection} describes the numerical solver returns.

We repeated the \(1{,}200\) attempted trials with the same seeds and probes,
assigning each retained key a fixed four-dimensional value vector. We
measured the largest readout difference over all included trials, dividing
each trial's difference by its reference readout range, defined as the
largest coordinate range across probes (Table~\ref{tab:tiebreak}). This
measurement covers the trials returned at each tolerance, whereas the
scalar diagnostic in Section~\ref{sec:verification} uses each regime's
worst \emph{gate-score} trial (Table~\ref{tab:evidence-map}).

\begin{table}[t]
\centering
\small
\caption{Independent minimum-norm projections of decrement and retained-key
refit. Tolerance rows use solver-returned pairs; raw rows use all completed
decrements. Gate-score differences are in kernel-similarity units. Readout
worst is the maximum across included trials of probe-wise error divided by
that trial's reference readout range, with fixed four-dimensional values.
Partition thresholds and solver-status details appear in
Appendix~\ref{app:projection}.}

\label{tab:tiebreak}
\setlength{\tabcolsep}{5pt}
\begin{tabular}{@{}l l c c c c@{}}
\toprule
regime & \(\varepsilon\) & solved & \shortstack{worst\\gate-score gap} & \shortstack{readout worst\\(\% readout range)} & \shortstack{partition\\match} \\
\midrule
Gaussian & none & --- & \(6.4\times10^{-3}\) & \(1.0\%\) & \(99\%\) \\
 & \(10^{-5}\) & \(299/300\) & \(5.9\times10^{-4}\) & \(0.2\%\) & \(100\%\) \\
 & \(10^{-4}\) & \(300/300\) & \(7.9\times10^{-4}\) & \(0.4\%\) & \(100\%\) \\
 & \(10^{-3}\) & \(300/300\) & \(3.4\times10^{-4}\) & \(0.1\%\) & \(100\%\) \\
\addlinespace
redundant & none & --- & \(9.6\times10^{-4}\) & \(51.7\%\) & \(74\%\) \\
 & \(10^{-5}\) & \(298/300\) & \(3.8\times10^{-3}\) & \(86.0\%\) & \(88\%\) \\
 & \(10^{-4}\) & \(300/300\) & \(5.2\times10^{-3}\) & \(12.6\%\) & \(98\%\) \\
 & \(10^{-3}\) & \(300/300\) & \(1.3\times10^{-3}\) & \(1.7\%\) & \(100\%\) \\
\addlinespace
MIMIC-IV & none & --- & \(1.1\times10^{-2}\) & \(5.3\%\) & \(99\%\) \\
 & \(10^{-5}\) & \(299/299\) & \(1.9\times10^{-3}\) & \(1.0\%\) & \(100\%\) \\
 & \(10^{-4}\) & \(299/299\) & \(3.5\times10^{-3}\) & \(1.0\%\) & \(100\%\) \\
 & \(10^{-3}\) & \(299/299\) & \(4.2\times10^{-3}\) & \(1.1\%\) & \(98\%\) \\
\addlinespace
learned keys & none & --- & \(1.1\times10^{-11}\) & \(0.0\%\) & \(100\%\) \\
 & \(10^{-5}\) & \(300/300\) & \(8.5\times10^{-11}\) & \(0.0\%\) & \(100\%\) \\
 & \(10^{-4}\) & \(300/300\) & \(5.7\times10^{-12}\) & \(0.0\%\) & \(100\%\) \\
 & \(10^{-3}\) & \(300/300\) & \(1.1\times10^{-8}\) & \(0.0\%\) & \(100\%\) \\
\bottomrule
\end{tabular}
\end{table}

We chose \(\varepsilon=10^{-3}\), the first tested tolerance above every
measured relative objective gap, which returned both projections on all \(1{,}199\)
completed trials; tighter tolerances sometimes returned no projection or
increased disagreement (Appendix~\ref{app:projection}). The worst readout
differences on Gaussian, redundant, and clinical keys fell from
\(1.0/51.7/5.3\%\) of range to \(0.1/1.7/1.1\%\), with partition agreement
of \(100/100/98\%\) (Table~\ref{tab:tiebreak}). A gate-score difference of
\(10^{-4}\)--\(10^{-3}\) remains as each side projects to its own objective
level.

\subsection{Sequential numerical reliability and checked execution}
\label{sec:journal-sequential-checks}

The fixed-cap result concerns a declared optimization problem; a finite-precision
implementation also needs a rule for accepting its returned state. We therefore
extend the standalone deletion audit with chained deletions and admissions.
This experiment evaluates numerical reliability of the established maintained
support-vector machinery, not a new optimization algorithm or a stronger
deletion theorem.

We schedule 80 synthetic trajectories: Gaussian and clustered near-duplicate
keys, initial sizes \(n_0\in\{64,128\}\), dimensions \(d\in\{6,64\}\), and ten
independent seeds (41000--41009) per condition. Each trajectory contains 32 cycles, each
consisting of a uniformly selected current-key deletion followed by one fresh
admission, for 5,120 scheduled operations. The cap remains
\(C=1/(0.4n_0)\), and the radial-kernel width remains
\(\sigma=2\sqrt{d/6}\). Gaussian keys have independent standard-normal coordinates. Near-duplicate
keys use eight fixed zero-mean Gaussian cluster centers with coordinate
standard deviation 2 and noise of standard deviation 0.02. Values are independent
four-dimensional uniform vectors on $[-1,1]$. Probes contain the retained and changed keys,
32 independent fixed anchors, and 16 jittered retained keys. This is a
synthetic dimension and maintenance-history extension, not a test on
high-dimensional transformer keys or task utility.

At every operation we compare against a fixed-cap refit on the current keys and separately
recompute mass error, coefficient-bound violations, and KKT stationarity from
the actual keys, coefficients, and stored offset. Let
\(t_i=K_{ii}-2(K\alpha)_i\). With the original numerical partition threshold
\(10^{-7}\), the stationarity terms are \(|\rho-t_i|\) on the margin,
\(\max(0,t_i-\rho)\) at the lower bound, and
\(\max(0,\rho-t_i)\) at the upper bound. A returned state passes only if the
maximum of these terms, mass error, and bound violations is at most
\(10^{-5}\). This is numerical acceptance at a stated tolerance, not equality
in exact arithmetic. A reference must pass the same check; agreement with a
reference alone does not establish state validity.

\paragraph{Initial audited run.}
The initial run completed 66 of 80 trajectories. Eight initial states failed
the check, and six trajectories stopped after a later replacement state also
failed it. Of the 5,120 scheduled operations, 4,390 ran, 4,384 ended in states
passing the check, and 730 were not run after these failures. There were 207
refit fallbacks: 204 raw candidates violated a coefficient bound by more than
\(10^{-5}\), and three failed stationarity. No raw update raised an exception.
Separately, 221 fresh reference states failed the state check. The largest raw
readout disagreement was 27.43\% of retained value range. All failures are included in the reported denominators.
Diagnostics localize infeasible updates to degenerate boundary events and
show that the original QP stopping settings can leave near-boundary
coefficients insufficiently accurate for the declared numerical partition.
These findings concern the implementation and reference recovery, not the
fixed-cap mathematical target.

\paragraph{Separately checked follow-up.}
The follow-up uses the same inputs and unchanged acceptance thresholds. A separate transactional wrapper copies the current state,
attempts the original update, and publishes the candidate only if the state
check passes. Otherwise it tries a direct QP refit at solver tolerance
\(10^{-12}\), with a predeclared \(10^{-13}\) second attempt if required. For
fixed coefficients it chooses a deterministic minimax KKT offset, writes that
offset into the actual state, recomputes cached gradients, and repeats the
unchanged state check. It does not clip, snap, renormalize, or regularize the
coefficients. If rescue cannot pass, it retains the previous valid state and
refuses the requested operation. The audit reference is a separately checked
strict refit through the same QP backend, not an independently implemented
optimizer.

All 80 follow-up trajectories and all 5,120 operations completed. Eight initial
states and 84 raw updates required rescue; the latter comprised 21 deletions
and 63 admissions, all with coefficient-bound violations. Every strict refit
passed at the first \(10^{-12}\) tier. No operation was refused or left unrun,
and no final state or strict reference failed the check. The worst published
state residual was \(8.61\times10^{-6}\); the worst strict-reference residual
was \(1.52\times10^{-6}\). The published gate-score difference had median,
95th percentile, and worst values of
\(5.76\times10^{-9}\), \(3.24\times10^{-7}\), and
\(6.04\times10^{-6}\), respectively. Published readout disagreement was at
most 0.5273\% of retained value range, or 0.9709\% under the separate
reference-readout-range normalization. The raw follow-up candidates still
include a 27.43\% readout tail before rejection and rescue.

\begin{table}[t]
\centering
\small
\caption{Both preserved runs use ten trajectories and 640 scheduled operations
per condition. All follow-up conditions completed 10/10 trajectories.
The Initial run and Checked follow-up column groups describe different runs. The final column is the maximum published readout disagreement divided by the
maximum coordinate range of the retained values; it is not decoded task
utility or the reference readout range. Initial rescues are additional to
update rescues.}
\label{tab:journal-sequential-guard}
\begin{tabular}{@{}lcrrrr@{}}
\toprule
& & \multicolumn{1}{c}{Initial run} & \multicolumn{3}{c}{Checked follow-up} \\
\cmidrule(lr){3-3}\cmidrule(l){4-6}
Regime & \((n_0,d)\) & \shortstack{Trajectories\\complete} &
\shortstack{Initial\\rescues} & \shortstack{Update\\rescues} &
\shortstack{Worst readout\\(\% value range)} \\
\midrule
Gaussian & (64,6) & 10/10 & 0 & 0/640 & 0.000297\% \\
Gaussian & (64,64) & 8/10 & 2 & 0/640 & 0.000056\% \\
Gaussian & (128,6) & 7/10 & 3 & 1/640 & 0.000164\% \\
Gaussian & (128,64) & 9/10 & 1 & 0/640 & 0.000526\% \\
Near-duplicate & (64,6) & 6/10 & 1 & 15/640 & 0.147706\% \\
Near-duplicate & (64,64) & 6/10 & 1 & 9/640 & 0.527320\% \\
Near-duplicate & (128,6) & 10/10 & 0 & 7/640 & 0.167766\% \\
Near-duplicate & (128,64) & 10/10 & 0 & 52/640 & 0.299647\% \\
\bottomrule
\end{tabular}
\end{table}

The follow-up tests a numerical controller around the unchanged maintained
solver. It neither repairs nor proves that solver; the evidence covers the
scheduled grid. The
trajectories differ after a rescue, so their raw update counts are not paired
measurements of the same evolving solver state. The 66 completed initial
trajectories have identical generated-input hashes in the follow-up. Raw
candidates, state checks, reference checks, rescue attempts, and all
denominators are recorded separately. Full transaction timing includes copying,
checking, and rescue, but measurements on a shared host support no timing
superiority claim. Original failures remain part of the reported result.

\section{Discussion}\label{sec:discussion}

The fixed-cap retained-key refit gives a requested deletion a concrete target.
The numerical evidence measures agreement with that target and the resulting
readout differences.

The sequential audit adds a numerical obligation: check the state that an
update actually returns. The raw implementation can return an infeasible
candidate without raising an exception, and a fresh numerical reference can
also fail the stated criterion. Checking both states separates completion,
numerical acceptance and reference agreement. The wrapper completes the finite
grid at unchanged thresholds, with strict refit rescues when necessary. It
remains a controller around the inherited solver; the degenerate event
handling has not been repaired or proved. Its reference uses the same QP
backend, with a separately recomputed state check.

Readout differences need their own measurement. Exact duplicates can admit
different coefficient optima, and distinct near-duplicates can yield sensitive
numerical solutions. Different values at similar addresses turn coefficient
redistribution into a readout difference. The earlier minimum-norm study
reduces its observed tail under an objective tolerance while leaving a
measurable gate-score difference. The sequential
study checks actual coefficients and offsets and reports readout differences
separately. Neither experiment establishes equality over every possible query.

The new trajectories are synthetic and operations within each trajectory are
correlated. They do not establish transfer to high-dimensional model keys,
decoded utility, membership-inference resistance or erasure of prior copies.
Changing the cap, re-encoding surviving keys or editing pretrained weights
defines another reference problem. The training path uses a separate
approximation and inherits neither the maintained numerical acceptance rule
nor the stored-inverse claim; its modeling and throughput results remain
bounded feasibility evidence in Appendix~\ref{app:training}. Direction-aware
handling of degenerate events and sufficient admission restrictions are
distinct next questions for implementation and theory.

\section{Conclusion}\label{sec:conclusion}

An entry can be removable from the current solved memory and become active
after a later admission. For the support-vector memory studied here, an
exactly zero coefficient certifies current-state removal, and active-key
deletion has a fixed-cap retained-key target under the stated assumptions.
The sequential audit preserves the raw implementation's failures; a separate
wrapper completes the same finite grid using explicit checks and rescues.
Those numerical results do not remove the limit shown by the three-key
construction: a current-state certificate does not automatically justify
permanent pruning across later admissions.

\section*{Statements and Declarations}

\subsection*{Data availability and reproducibility}

The manuscript source bundle contains the TeX sources and plotted figures.
The sequential protocols, solver settings, and outcome denominators are
reported in Section~\ref{sec:journal-sequential-checks}. Scripts and complete
source-free synthetic records are supplied as Supplementary Material,
\path{sv_attention_sequential_evidence_v1.zip}, which is available to reviewers
through the submission. Code and shareable evidence are available at
\url{https://github.com/VyLabs-AI/sv-attention}; the reviewer supplement also
supplies the code and synthetic results needed to check the sequential audit.
MIMIC-IV v3.1 requires
credentialed PhysioNet access under its data use agreement. Clinical processing
was local. The accompanying clinical artifacts comprise aggregate results and
the relative-hour single-stay figure; patient-level source records and derived
caches are not included.

\subsection*{Funding}

This research received no funding.

\subsection*{Competing interests}

Vishwajith Ramesh is CEO and Founder of Vy Labs, Inc. He declares no other
competing financial or non-financial interests and no editorial role at
\emph{Machine Learning}.

\subsection*{Ethics and consent}

The clinical analyses reuse MIMIC-IV v3.1 under its credentialed data-use
agreement. The resource's documentation reports that the Institutional Review
Board at Beth Israel Deaconess Medical Center reviewed the collection and
creation of the resource, waived informed consent, and approved data sharing
(\url{https://physionet.org/content/mimiciv/3.1/}). These are the approvals
reported for the source resource. Clinical processing for this study was
local. No separate project-specific ethics approval or consent waiver was
obtained for this secondary analysis. Patient-level source records and derived caches are not distributed
with the submission.

\bibliographystyle{tmlr}
\bibliography{references}

\appendix
\section{Rare-group retention under context selection}\label{app:selection}

We wanted to see which useful entries survived when we reduced the
context. The boundary fit favors keys that are atypical relative to the
others, so we expected it to help when informative tokens were rare and
the rest of the context was redundant. We gave each method the same
number \(k\) of tokens to retain, then tested rare-group recall and
retention of clinical event hours. To isolate the effect of selection,
we assigned equal coefficients to the retained keys and used the same
similarity readout for every method.

We first tested selection under skewed redundancy. Contexts contained six
singleton groups and six groups of ten near-duplicates, and every policy kept
the support set's mean token count of \(56\%\). Rare-group recall, defined as
the fraction of held-out queries near singleton-group centers whose returned
group is correct, was \(0.861\) for support selection, against \(0.319\) for an
attention-mass proxy in the style of Heavy-Hitter Oracle~\citep{h2o} that ranks
keys by summed similarity to all evaluation queries, \(0.472\) for uniform
sampling, and \(0\) for recency. The support-vector policy does not inspect the
queries.

In the clinical comparison, an unusual oxygen measurement could make an
hour both an event and a likely support point. We therefore checked
whether selectors retained event hours when they could not see that
measurement. We defined events by unstandardized hourly SpO\(_2<90\%\),
then removed SpO\(_2\) from every selector's input, leaving five channels.
Across the \(718\) qualifying stays at a matched mean
\(32\%\) token budget, support selection retained \(0.464\) of event hours
against \(0.225\) for the similarity-density proxy, with recency at \(0.409\)
and random selection at \(0.333\). The paired mean difference was \(0.239\);
resampling the \(718\) paired stay-level differences \(10{,}000\) times gave a
\(95\%\) percentile interval of \(0.194\)--\(0.284\), treating stays as
independent. When the selector could also see the vitals defining the
event, the same-channel version of this probe gave \(0.80\) against
\(0.05\). Excluding the event-defining measurement therefore reduced the
advantage, although support selection still retained more event hours
than the similarity-density proxy. Figure~\ref{fig:clinical-timeline}
shows the timing of retained hours for one stay from the same-channel
experiment.

We also evaluated a negative control with ten pairs of distinct keys.
Recall@1 was \(0.970\) without the gate and \(0.955\) with it, and the
support fraction was \(1.00\). Every key was active, so no keys were
removed; recall came from the similarity readout. The selection
benefits we observed were therefore specific to the redundancy and
atypicality probes above. Moreover, in a separate six-hour deterioration
prediction task, we obtained an AUROC of \(0.696\), below an LSTM at
\(0.783\) and softmax attention at \(0.751\). Preserving atypical hours
can help keep event records available, but it does not by itself optimize
prediction of local trends.

\begin{figure}[h]
\centering
\includegraphics[width=0.94\linewidth]{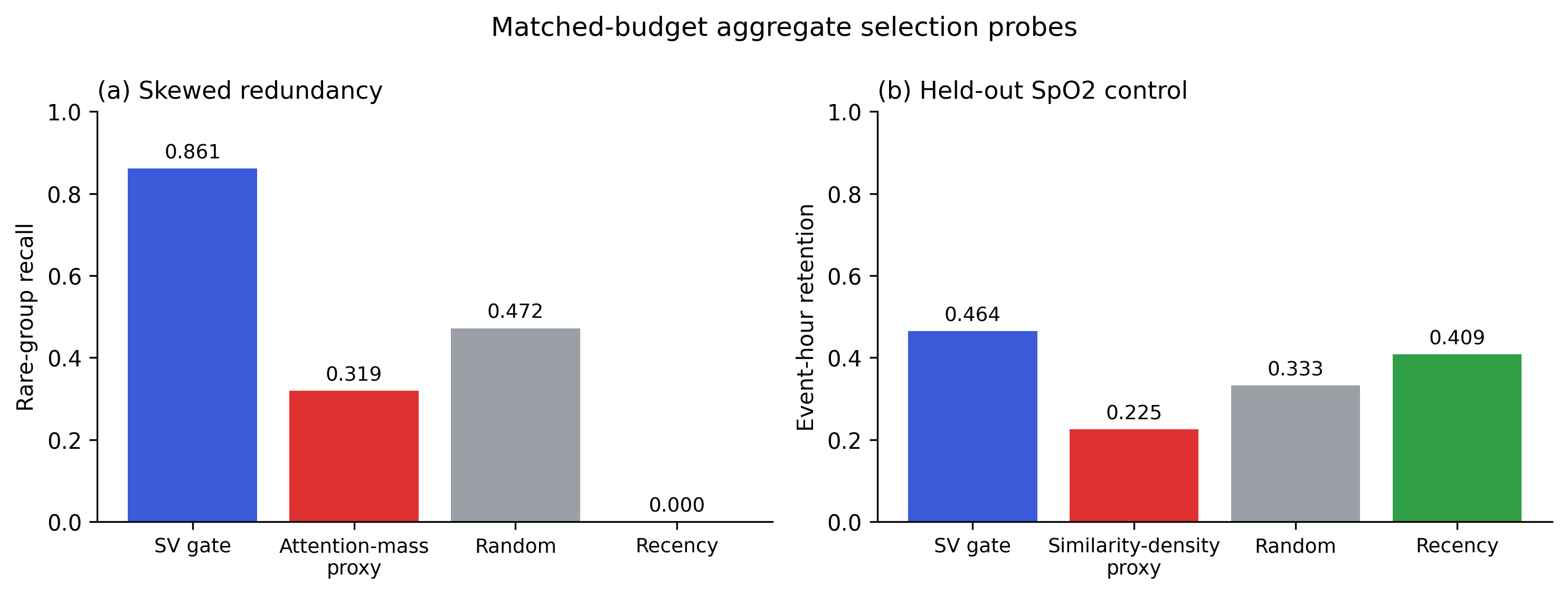}
\caption{Selection at matched token budgets. (a) Rare-group recall under skewed
redundancy. (b) Event-hour retention in the held-out SpO$_2$ control: every
selector excludes the event-defining channel. Across 718 stays at a matched
mean 32\% token budget, support selection retains 0.464 of event hours versus
0.225 for the similarity-density proxy. The outcome is selection of event
hours; clinical outcomes are not evaluated.}
\label{fig:selection}
\end{figure}

\begin{figure}[p]
\centering
\includegraphics[width=\linewidth]{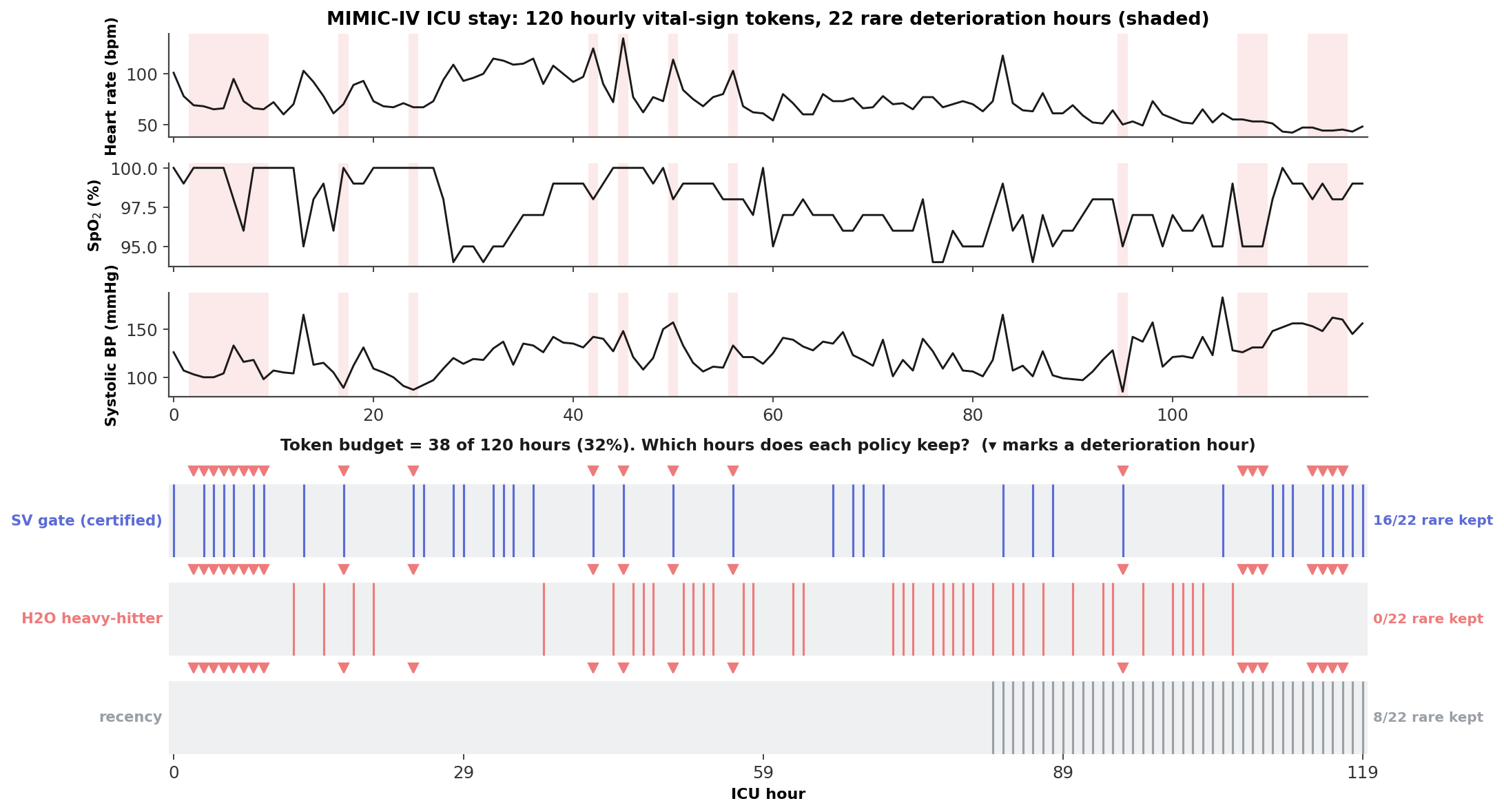}
\caption{Within-stay selection on MIMIC-IV at a matched budget of 38 out of
120 hourly tokens. Top: heart rate, SpO$_2$, and systolic blood pressure;
shading marks 22 hours flagged by the composite vital-sign threshold rule.
Bottom: retained hours for each policy. Support selection retains 16 flagged
hours, the similarity-density proxy retains 0, and recency retains 8. All
selectors use the six vital channels, including those defining the flags.
The panel labeled H2O denotes the similarity-density proxy. The held-out-channel
cohort control is reported in Figure~\ref{fig:selection}.}
\label{fig:clinical-timeline}
\end{figure}

\FloatBarrier
\section{Training the support-vector gate}\label{app:training}

Before training models with the gate, we checked its gradients. The
maintained custom gradient
matched finite differences for the coefficients, gate scores, and full
readout. A trainable student layer reduced mean-squared error against a
fixed teacher mapping by \(99\%\) in \(200\) optimizer steps. End-to-end
gradient checks also passed for the batched ridged path. Those
checks validate the stabilized approximation used during training; they
do not validate the unregularized fp64 coefficients.

We first trained a hybrid at \(3.22\)M parameters on the first \(5\)
million bytes of enwik8. All seven paired seeds favored the hybrid: mean
best-validation bits per character was
\(2.178\) against \(2.383\) for a sliding-window Transformer whose
visible-token count matched the hybrid's measured active state plus local
window. The mean paired relative improvement was \(8.6\%\), and a two-sided
paired \(t\)-test on the absolute differences gave \(p=0.001\). On TinyStories
the hybrid was lower on all three seeds, directionally positive but not
significant (\(p=0.057\))~\citep{tinystories}.

When we increased the budgets, however, the comparison reversed. At
matched \(6{,}000\)-step budgets, the hybrid was \(2.5\%\) worse on
average over five seeds at \(10\)M parameters, with high variation
between seeds. One \(32\)M run was also substantially behind. We
observed slower and less stable optimization of the gate at these
budgets, although we did not train the models to convergence and cannot
conclude how their performance would compare after full optimization.

At the \(3.22\)M configuration, we measured \(9{,}125\) tokens per second
for the combined forward, loss, and backward computation on an Apple M3
Ultra. This mixed path was \(35.8\times\) slower than the repository's
unfused softmax baseline. Maintained fp64 inference, which explicitly
updates coefficient groups, is slower still. We could therefore train
and evaluate small models, but the throughput and optimization results
leave substantial work before the gate could be part of a competitive
architecture or accelerator kernel.

\section{Protocols}\label{app:protocols}
\begingroup\setstretch{1.0}

For each trial of the audit at fixed \(C\), we built a context, set
\(C\) once, and solved the gate in fp64. We then removed the same two
keys through the reverse decrement and by refitting the surviving keys
at the same \(C\), keeping \(C\) unchanged as the context became
smaller. We compared gate scores on retained, removed, and perturbed
probes and recorded the maximum difference. We used scores because
near-duplicates can make the coefficients numerically sensitive even
when the exact optimum is unique, and evaluated readout agreement as
described in the main text. For the decay baseline, we scaled both
removed keys' coefficients by \(\gamma=0.01\) and used the same
measurement. The regimes were Gaussian keys (\(d=6\)), redundant
near-duplicates, real MIMIC-IV hourly vitals~\citep{mimiciv}, and a
trained model's learned key projections. We attempted \(300\) trials
per regime and stored all \(1{,}200\) statuses.

In the clinical trial where the decrement could not complete because the
margin set was empty, we used the refit that had already been computed.
The archive records the failed decrement and the reference solve. It
contains execution records and scalar comparisons; the missing keys,
coefficient arrays, and offsets prevent the independent state checks
described in Section~\ref{sec:verification}.

We also recorded the condition number of the retained Gram matrix and
agreement between the coefficient groups for the comparisons in
Section~\ref{sec:verification}. To translate the coefficient differences
into readout differences, we evaluated both vectors through the normalized
readout on the same probes. We drew the scalar values from an
independently seeded random stream so that this diagnostic would not
consume draws from the audited sequence. The contexts, deleted indices,
and probe jitter were bit-identical to the run without the diagnostic,
and every reported deviation was unchanged. The seed fixed the values
before any deviation was observed; we did not choose them in response
to the errors.

For verification, we used the maintained fp64 CPU solver, updating its
stored inverse after each single-key change and periodically recomputing
it from scratch. For training, we found coefficient partitions using
\(80\) iterations of accelerated projected gradient in single
precision, then used a batched linear solve with a \(10^{-3}\) ridge and
differentiated it with autograd. We used an RBF kernel with chunk length
\(64\). The constrained problem requires \(nC\geq1\), so the causal
implementation runs without the gate when fewer than
\(\lceil1/C\rceil+2\) keys are available and issues no certificate there.

We processed MIMIC-IV v3.1 locally under its credentialed data use
agreement~\citep{mimiciv}. The released clinical material contains
aggregate results and the relative-hour case figure
(Figure~\ref{fig:clinical-timeline}), without patient-level caches. For
the control that withheld the event-defining channel, we labeled events
using unstandardized hourly SpO\(_2<90\) after forward and backward
filling within each stay and cohort-median imputation. We removed the
standardized SpO\(_2\) column before computing the five-dimensional
gate and all baselines at the matched budget. To obtain the stay-level
percentile interval, we bootstrapped the \(718\) paired retention
differences \(10{,}000\) times with seed \(0\).

\subsection{Objective-tolerance projection diagnostics}\label{app:projection}

To solve Equation~\ref{eq:projection}, we symmetrize \(K\) and try
CLARABEL followed by SCS, accepting the statuses \texttt{optimal} or
\texttt{optimal\_inaccurate}. We then clip coefficients to \([0,C]\)
and renormalize their sum. The ``solved'' counts report pairs returned
by this procedure; we did not independently check feasibility after
postprocessing. Each side uses its own achieved objective plus
\(\varepsilon\max(1,|f|)\) as its cap. For the relative gap in
Table~\ref{tab:projection-diagnostics}, we use the refit's
\(\max(1,|f|)\) as the denominator.

\begin{table}[t]
\centering\small
\caption{Diagnostics for the projection study. Movement is the median largest
coefficient change at \(\varepsilon=10^{-3}\), with both sides rounding to
the displayed values; raw disagreement compares the
unprojected decrement and retained-key refit. Rank correlation relates relative
objective gap to raw gate-score deviation.}
\label{tab:projection-diagnostics}
\begin{tabular}{@{}l c c c@{}}
\toprule
Diagnostic & Gaussian & redundant & MIMIC-IV \\
\midrule
Maximum relative objective gap & \(1.7\times10^{-5}\) & \(5.1\times10^{-5}\) & \(1.2\times10^{-4}\) \\
Spearman rank correlation & \(+0.36\) & \(+0.69\) & \(+0.56\) \\
Median projection movement & \(1.5\times10^{-2}\) & \(6.0\times10^{-2}\) & \(2.1\times10^{-2}\) \\
Median raw disagreement & \(1.6\times10^{-9}\) & \(2.2\times10^{-6}\) & \(1.8\times10^{-9}\) \\
\bottomrule
\end{tabular}
\end{table}

At \(\varepsilon=10^{-6}\), no projection was returned for the Gaussian trial
with the largest raw deviation or for six near-duplicate trials. The clinical
worst gate-score disagreement was \(5.3\times10^{-2}\), nearly five times the
raw worst case. At \(10^{-8}\) and \(10^{-10}\), the near-duplicate worst
readout difference exceeded \(70\%\) of range. Table~\ref{tab:tiebreak} reports
\(10^{-5}\) and the two looser tolerances.

Both sides use the same coefficient constraints and objective, so their
objective sublevel sets are nested and overlap when both are nonempty.
Even so, tight caps can make the numerical solve sensitive. At
\(10^{-3}\), we reduced the median gate-score deviation by about an
order of magnitude in every regime, while the coefficient movements in
Table~\ref{tab:projection-diagnostics} were much larger than the
original disagreement. Each minimum-norm point depends on a different
objective level, which explains why a residual remains even after both
solutions move substantially. We obtained readout disagreement below
\(2\%\) of range in every regime, paying for one additional solve
per side.

The scalar readout diagnostic in Section~\ref{sec:verification} gives gate-score
range fractions of \(2.8/3.0/2.8\%\) on the Gaussian, redundant, and clinical
worst gate-score trials. In the four-dimensional projection study, the worst
raw readout trial differs from the worst gate-score trial for Gaussian and
redundant keys. Partition agreement there uses the fixed coefficient threshold
\(10^{-6}\), rather than each maintained solver's recorded partition.

\endgroup

\end{document}